\documentclass[11pt]{article}

\PassOptionsToPackage{table}{xcolor}

\usepackage[preprint]{acl}

\usepackage{times}
\usepackage{latexsym}

\usepackage[T1]{fontenc}

\usepackage[utf8]{inputenc}

\usepackage{microtype}

\usepackage{inconsolata}

\usepackage{graphicx}
\usepackage{booktabs}
\usepackage{multirow}
\usepackage{hhline}
\usepackage{amsmath, amssymb, amsthm, algorithm, algorithmic}

\newtheorem{definition}{Definition}
\newtheorem{theorem}{Theorem}
\newtheorem{proposition}{Proposition}

\title{Gated Differentiable Working Memory \\ for Long-Context Language Modeling}

\author{
  Lingrui Mei$^{1, 2, \dagger}$ \quad
  Shenghua Liu$^{1, 2, \dagger}$ \quad
  Yiwei Wang$^{3}$ \quad
  Yuyao Ge$^{1, 2, \dagger}$ \quad
  Baolong Bi$^{1, 2, \dagger}$ \\
  \textbf{Jiayu Yao$^{1, 2}$ \quad
  Jun Wan$^{4}$ \quad
  Ziling Yin$^{1, 2, \dagger}$ \quad
  Jiafeng Guo$^{1, 2, \dagger}$ \quad
  Xueqi Cheng$^{1, 2, \dagger}$} \\[1em]
  $^{1}$Institute of Computing Technology, Chinese Academy of Sciences \\
  $^{2}$University of Chinese Academy of Sciences \\
  $^{3}$University of California, Merced \quad
  $^{4}$UBS AG \\[0.5em]
  \texttt{\{meilingrui25b, liushenghua\}@ict.ac.cn}
}

\begin{document}
\maketitle
\renewcommand{\thefootnote}{\fnsymbol{footnote}}
\footnotetext[2]{Also affiliated with: (1) Key Laboratory of Network Data Science and Technology, ICT, CAS; (2) State Key Laboratory of AI Safety.}

\begin{abstract}
Long contexts break transformers: attention scores dilute across thousands of tokens, critical information gets lost in the middle, and the model cannot adapt to novel patterns at inference time. Recent work on test-time adaptation addresses this by maintaining a form of \emph{working memory}---transient parameters updated on the current context---but existing approaches employ \emph{uniform} write policies that waste computation on low-value regions and suffer from high gradient variance across semantically heterogeneous contexts.
In this work, we reframe test-time adaptation as a budget-constrained memory consolidation problem, asking: \emph{given limited computational budget, which parts of the context should be consolidated into working memory?}
We propose \textsc{Gdwm} (\textbf{G}ated \textbf{D}ifferentiable \textbf{W}orking \textbf{M}emory), a framework that introduces a Write Controller to gate the memory consolidation process.
Our controller estimates \emph{Contextual Utility}---an information-theoretic measure quantifying how much each region depends on long-range context---and allocates gradient steps accordingly, subject to a coverage constraint that ensures global representation.
Experiments on ZeroSCROLLS and LongBench v2 benchmarks demonstrate that \textsc{Gdwm} achieves comparable or superior performance with 4$\times$ fewer gradient steps compared to uniform baselines, establishing a new efficiency-performance Pareto frontier for test-time adaptation.
\end{abstract}

\section{Introduction}

\begin{figure}[t]
    \centering
    \includegraphics[width=\columnwidth]{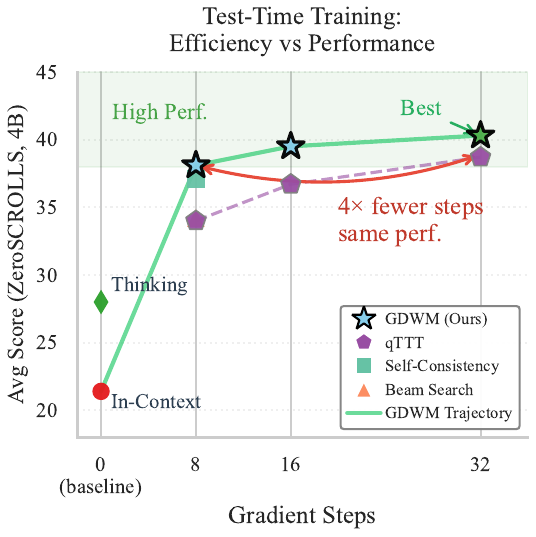}
    \caption{\textbf{Efficiency vs Performance on ZeroSCROLLS (\textsc{Qwen3-4B}).} \textsc{Gdwm} achieves comparable performance to qTTT-32 with only 8 gradient steps (4$\times$ fewer), establishing a new Pareto frontier. Context-aware budget allocation enables faster convergence than uniform or sampling-based alternatives.}
    \vspace{-10pt}
    \label{fig:teaser}
\end{figure}

Large Language Models (LLMs) struggle with long contexts: attention scores dilute, and models miss critical information in central positions (e.g., ``Lost-in-the-Middle'') \citep{liu2023lostmiddlelanguagemodels, hsieh2024rulerwhatsrealcontext, du2025contextlengthhurtsllm, yao2025spotlighthiddenbiasundermining}. Recent work mitigates such long-context failures by equipping LLMs with \emph{working-memory}—either as lightweight, differentiable \emph{fast-weight} \citep{ba2016usingfastweightsattend} states updated online via self-supervised gradients, or as explicit memory modules that incrementally read, consolidate, and overwrite salient information during inference \citep{yu2025memagentreshapinglongcontextllm,bansal2025letsnotjustthings, hong2025context, inplace_test_time_training_2025, xu2025amemagenticmemoryllm, ye2025agentfoldlonghorizonwebagents}.

\begin{figure*}[t]
    \centering
\includegraphics[width=\textwidth]{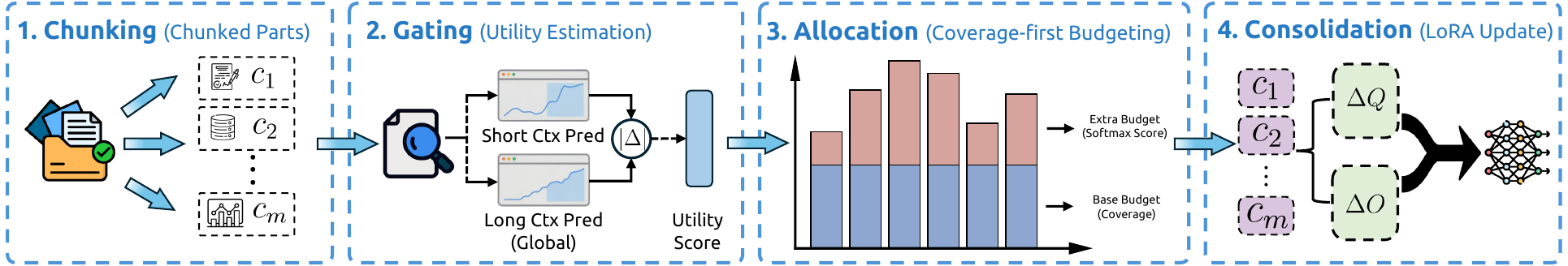}
    \caption{\textbf{High-level overview of \textsc{Gdwm}.} The framework proceeds in four stages: \textbf{Chunk} the input into fixed-size units (approximating semantic segments), \textbf{Gate} each chunk via Contextual Utility (CPMI-based divergence), \textbf{Allocate} gradient budget proportionally subject to coverage constraints, and \textbf{Consolidate} into LoRA adapters.}
    \vspace{-3pt}
    \label{fig:overview}
\end{figure*}

However, current approaches predominantly rely on \emph{uniform write policies}, stochastically sampling tokens across the entire context. This is suboptimal: information density is non-uniform, wasting budget on low-value regions, and global sampling exacerbates gradient variance by aggregating conflicting updates from semantically disparate regions.

We propose \textsc{Gdwm} (\textbf{G}ated \textbf{D}ifferentiable \textbf{W}orking \textbf{M}emory), a framework that recasts test-time adaptation as budget-constrained memory consolidation. The key question becomes: \emph{given limited compute, which parts of the context should be consolidated into working memory?} \textsc{Gdwm} answers this via a Write Controller that estimates \emph{Contextual Utility}---the divergence between long-context and short-context predictions---and allocates gradient steps to regions where long-range dependencies are most critical. The framework is mechanism-agnostic: it provides a principled write policy that can be layered on top of any test-time adaptation architecture.

First, we formalize test-time adaptation as a budget-constrained resource allocation problem, cleanly separating the memory mechanism (how to update) from the write policy (where and how much to update). Second, we propose a chunk-wise, budget-aware algorithm driven by \emph{Contextual Utility}---an information-theoretic measure grounded in Conditional Pointwise Mutual Information (CPMI) that identifies high-value long-range dependencies. Third, we prove that chunk-restricted sampling reduces gradient variance by eliminating inter-chunk interference via the Law of Total Variance. Extensive evaluation on ZeroSCROLLS and LongBench v2 across three model scales (1.7B, 4B, 8B) demonstrates that \textsc{Gdwm} achieves comparable or superior performance with 4$\times$ fewer gradient steps, yielding significant gains on sparse-information tasks (up to +12.7\% on Qasper, up to +11.2\% on GovReport) while achieving 39\% wall-clock speedup. Ablation studies confirm that CPMI-based selection, coverage constraints, and chunk-based processing are all essential components. We further show that \textsc{Gdwm} outperforms alternative test-time scaling strategies (self-consistency, beam search), and provide a theoretical analysis linking optimal chunk size to task-specific evidence spans.

\section{Related Work}

\begin{figure*}[t]
    \centering
\includegraphics[width=\textwidth]{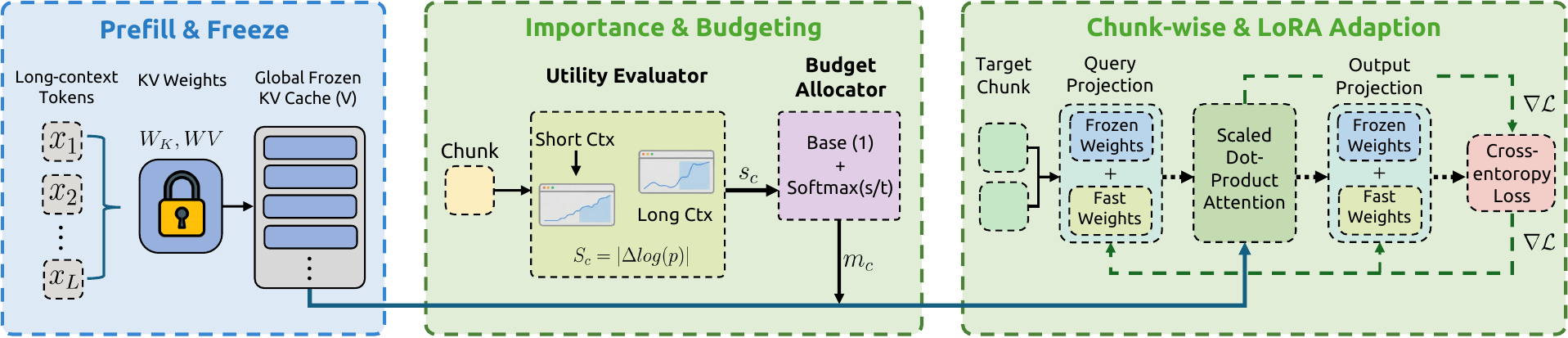}
    \caption{\textbf{Technical details of \textsc{Gdwm}.} Left: Prefill-and-freeze KV cache enables efficient chunk-wise processing. Middle: Contextual Utility is computed as CPMI between global and local predictions, then converted to budget weights via softmax allocation. Right: LoRA adapters on $W_Q$/$W_O$ projections are updated through chunk-wise next-token prediction loss.}
    \vspace{-3pt}
    \label{fig:gdwm}
\end{figure*}

\paragraph{Context Engineering}
Context engineering optimizes information payloads for LLMs through prompt design, in-context learning, and chain-of-thought prompting \citep{Wei2022ChainThought, Brown2020LanguageModelsa}. Retrieval-augmented generation (RAG) enhances knowledge access \citep{Lewis2020RetrievalAugmenteda, Karpukhin2020DensePassage}, while compression and hierarchical processing address quadratic scaling limitations \citep{Li2023CompressingContext, Vaswani2017AttentionAll}. Recent work tackles robustness via context-aware decoding and representation engineering \citep{He2024ContextSteering, Zhou2023ContextFaithful}.

\paragraph{LLM Memory}
Memory in LLMs addresses bottlenecks from the KV cache and model weights \citep{Liu2023ScissorhandsExploiting, Dao2022FlashattentionFast}, with solutions including paging \citep{Kwon2023EfficientMemoryb}, compression \citep{Zhang2024LorcLow}, and dynamic eviction \citep{Zhang2023HeavyHitter, Yuan2025WeightedkvAttentiona}. Agent memory research extends this by enabling persistent storage beyond context windows through hierarchical memory architectures \citep{Fang2024UnimemTowards, Zhong2023MemorybankEnhancing} and structured memory types \citep{Han2024LlmMulti, Terranova2025EvaluatingLong}. Contemporary frameworks implement consolidation, updating, and forgetting operations \citep{Park2023GenerativeAgents, Cao2025RememberRefine}, though challenges remain in autonomous management and catastrophic forgetting \citep{Xu2025MemAgentica, Guo2023EmpoweringWorking}.

\paragraph{Test-Time Adaptation}
Recent work treats deployment as an optimization phase, updating model states at inference time via self-supervised objectives \citep{Niu2022EfficientTMD, Tang2023NeuroModulatedHLA}. Within LLMs, gradient-based test-time training adapts to task instances through neighbor retrieval, stream processing, or active selection \citep{hardt2024testtimetrainingnearestneighbors, muhtar2024streamadapterefficienttesttime, huebotter2025efficientlylearningtesttimeactive, akyurek2025surprisingeffectivenesstesttimetraining}. Parallel efforts optimize test-time compute allocation through principled scaling and meta-learned control \citep{Snell2024ScalingLTB, Qu2025OptimizingTCA}, while complementary strategies compress long contexts via selective augmentation and learned prompt compression \citep{xu2023recompimprovingretrievalaugmentedlms, jiang2024longllmlinguaacceleratingenhancingllms}.

\section{Problem Formulation}

We view the process of adapting an LLM to a long context $X = (x_1, x_2, \ldots, x_L)$ as a Budget-Constrained Memory Consolidation problem.

\paragraph{Fast Weights.}
Let $\theta$ denote the parameters of a pre-trained LLM. We partition $\theta$ into frozen parameters $\theta_f$ (including the KV cache computation) and adaptable ``fast weights'' $\phi$ (e.g., LoRA adapters \citep{hu2021loralowrankadaptationlarge} on Query projections).

\paragraph{Budget-Constrained Optimization.}
We partition the input context $X$ into $M$ fixed-size chunks $\{C_1, C_2, \ldots, C_M\}$. The goal is to determine an optimal integer allocation schedule $\mathbf{k} = (k_1, \ldots, k_M)$, where $k_c$ denotes the number of minibatch gradient updates performed on chunk $C_c$.

Since the downstream task loss $\mathcal{L}_{\text{task}}$ is unknown during inference, we employ the self-supervised next-token prediction loss as a surrogate objective. The optimization problem is formulated as:
\vspace{-3pt}
\begin{equation}
    \label{eq:objective}
    \begin{split}
    \min_{\{k_c\}_{c=1}^M} \sum_{c=1}^M &\mathbb{E}_{t \sim \mathcal{U}(\mathcal{I}_c)} \\\
    &\quad \bigl[ -\log P(x_t \mid x_{1:t-1}; \phi(\mathbf{k}), \theta_f) \bigr]
    \end{split}
\end{equation}
where $\mathcal{I}_c$ denotes the set of positions in chunk $C_c$, and $\mathbb{E}_{t \sim \mathcal{U}(\mathcal{I}_c)}$ denotes the expectation over positions $t$ sampled uniformly from $\mathcal{I}_c$. The constraints are:
\vspace{-3pt}
\begin{align}
    \label{eq:budget_constraint}
    &\sum_{c=1}^M k_c \le K_{\text{total}} \quad \text{(Total Budget)} \\
    \label{eq:coverage_constraint}
    &k_c \ge k_{\min}, \; \forall c \in \{1, \ldots, M\} \quad \text{(Coverage)}
\end{align}
Here, $\phi(\mathbf{k})$ is an \emph{implicit function} of the allocation $\mathbf{k}$, defined by the gradient descent process: starting from $\phi^{(0)} = 0$, we sequentially perform $k_c$ gradient updates on each chunk $C_c$, yielding $\phi(\mathbf{k}) = \phi^{(\sum_c k_c)}$. Since exactly solving this bi-level problem is intractable, we propose a heuristic approximation in Section~\ref{sec:method} using Contextual Utility as a proxy for the gradient contribution of each chunk.

\section{Method}
\label{sec:method}

Our method, \textsc{Gdwm}, solves the optimization problem defined in Equation~\ref{eq:objective} via a four-stage process: (1) Chunking, (2) Gating (Importance Estimation), (3) Allocation, and (4) Consolidation. Figure~\ref{fig:overview} illustrates the high-level pipeline, and Figure~\ref{fig:gdwm} details the technical components.

\subsection{Memory Interface}

\textsc{Gdwm} is mechanism-agnostic: the Write Controller operates as a policy layer orthogonal to the memory mechanism. While $\phi$ can be any differentiable parameters (linear, MLP, etc.), we focus on LoRA adapters for efficiency.

While \textsc{Gdwm} is mechanism-agnostic and can be layered on top of any adaptation architecture, we instantiate it efficiently by attaching LoRA adapters to projection matrices while freezing the KV cache (see Section~\ref{sec:implementation} for details). The fast weights $\phi$ are initialized to zero (identity mapping) at the start of each sequence.

\subsection{Contextual Utility Estimation}

To efficiently solve the allocation problem, we need a proxy for the gradient contribution of each chunk. We introduce \emph{Contextual Utility}, an information-theoretic measure grounded in the cognitive notion of surprisal.
The core intuition is simple: if a token $x_t$ can be predicted easily using only local context, consolidating it into fast weights yields low utility. Conversely, if $x_t$ is unpredictable locally but predictable globally, it represents a high-value long-range dependency.

\begin{definition}[Contextual Utility]
Let $\mathcal{I}_c$ denote the set of token positions in chunk $C_c$. The Contextual Utility is defined as the average surprisal divergence:
\vspace{-3pt}
\begin{align}
    \label{eq:utility}
    U(C_c) &= \frac{1}{|\mathcal{I}_c|} \sum_{t \in \mathcal{I}_c} \Delta_t \nonumber \\
    \Delta_t &= \bigl| \log P(x_t \mid x_{1:t-1}) \nonumber \\
             &\quad - \log P(x_t \mid x_{t-n:t-1}) \bigr|
\end{align}
where $P(x_t \mid x_{1:t-1})$ is the conditional probability using full context (``global prediction'') and $P(x_t \mid x_{t-n:t-1})$ uses only the recent $n$ tokens (``local prediction'').
\end{definition}

We take the absolute value to capture both directions: positions where long context \emph{helps} ($P_{\text{full}} > P_{\text{local}}$) and where it \emph{conflicts} ($P_{\text{full}} < P_{\text{local}}$). Both indicate regions requiring gradient-based calibration. Empirically, $|\Delta_t|$ outperforms $\max(0, \Delta_t)$ by 3-5\% (see Appendix~\ref{app:absolute_value} for analysis).

\paragraph{Interpretation.}
The term inside the sum represents the \emph{Surprisal Divergence}. When $P_{\text{full}} \approx P_{\text{local}}$, the long context provides no additional information (utility is zero). High utility indicates regions where global dependencies are critical for prediction. 

\paragraph{Information-Theoretic Grounding.}
The quantity $\Delta_t$ has a precise information-theoretic interpretation: it equals the absolute value of 
\emph{Conditional Pointwise Mutual Information} (CPMI) between the token $x_t$ and the long-range prefix $x_{1:t-n}$, conditioned on the local context $x_{t-n:t-1}$. 
Formally, $\text{CPMI}(A; B \mid C) = \log P(A \mid B, C) - \log P(A \mid C)$ measures the information that $B$ provides about $A$ beyond what $C$ already provides. 
High $|\Delta_t|$ indicates positions where the model's prediction strongly depends on (or conflicts with) long-range context---precisely where gradient-based memory consolidation is most valuable.

\subsection{Budget-Aware Allocation Policy}

Exact solution of the integer programming problem in Eq.~(\ref{eq:objective}) is intractable. We propose a heuristic approximation based on \textbf{Coverage-First Softmax Allocation}.

\textbf{Step 1: Satisfy Coverage Constraint.}
Allocate $k_{\min}$ steps to every chunk:
\vspace{-3pt}
\begin{equation}
    k_c \leftarrow k_{\min}, \quad \forall c \in \{1, \ldots, M\}
\end{equation}
\textbf{Handling Infeasible Budgets.} If the total budget $K_{\text{total}} < M \cdot k_{\min}$, the coverage constraint cannot be strictly satisfied. In such cases, we relax the constraint by allocating $k_{\min}$ steps to the top-$\lfloor K_{\text{total}} / k_{\min} \rfloor$ chunks with the highest utility $U(C_c)$, and 0 to others. This fallback ensures that limited compute is invested in the most critical regions first.

\textbf{Step 2: Distribute Remaining Budget.}
The remaining budget $K_{\text{rem}} = \max(0, K_{\text{total}} - M \cdot k_{\min})$ is distributed based on normalized utility:
\vspace{-3pt}
\begin{align}
    \label{eq:softmax_weight}
    w_c &= \frac{\exp(U(C_c)/\tau)}{\sum_{j=1}^M \exp(U(C_j)/\tau)} \\
    \label{eq:allocation}
    k_c &\leftarrow k_c + \left\lfloor K_{\text{rem}} \cdot w_c \right\rfloor
\end{align}
To ensure the full budget is utilized, we apply the \emph{Largest Remainder Method} (see Appendix~\ref{app:allocation}) to distribute residual steps, guaranteeing $\sum_c k_c = K_{\text{total}}$ exactly.
The temperature $\tau$ controls allocation sharpness. When $\tau \to 0$, all extra budget concentrates on the highest-utility chunk; when $\tau \to \infty$, allocation becomes uniform (ignoring utility); and $\tau = 1.0$ provides balanced allocation that proves empirically optimal.
This policy directs the optimizer's focus to high-utility regions while maintaining global awareness through the coverage constraint---hence the term ``gated'' in our framework name.

\subsection{Structured Memory Consolidation}

Finally, we execute updates via chunk-restricted sampling. For each chunk $C_c$, we perform $k_c$ gradient descent steps. In each step, we uniformly sample a minibatch of positions $\mathcal{I} \subset \mathcal{I}_c$ and compute the loss:
\begin{equation}
    \mathcal{L} = -\frac{1}{|\mathcal{I}|} \sum_{i \in \mathcal{I}} \log P(x_i \mid x_{1:i-1}; \phi, \theta_f)
\end{equation}
The KV cache can be precomputed and reused across all $k_c$ steps for efficiency (see Section~\ref{sec:implementation}).

\begin{table*}[t]
\centering
\resizebox{\textwidth}{!}{%
\begin{tabular}{@{}l c cccccc c@{}}
\toprule\toprule
\multirow{2}{*}{\textbf{\textsc{Method}}} & \multirow{2}{*}{\textbf{\textsc{Steps}}} & \multicolumn{2}{c}{\textbf{\textsc{Summarization}}} & \multicolumn{2}{c}{\textbf{\textsc{Question Answering}}} & \textbf{\textsc{Comprehension}} & \textbf{\textsc{Reasoning}} & \multirow{2}{*}{\textit{Avg.}} \\
\cmidrule(lr){3-4} \cmidrule(lr){5-6} \cmidrule(lr){7-7} \cmidrule(lr){8-8}
& & GovReport & QMSum & Qasper & NarrativeQA & Quality & Musique & \\
\midrule
\multicolumn{9}{c}{\textsc{\textbf{Qwen3-1.7B}}} \\
\cmidrule(l{2pt}r{2pt}){1-9}
In-context & -- & 22.2 & 6.4 & 25.8 & 14.9 & 48.1 & 11.8 & 21.5 \\
Thinking & -- & 21.5 & 7.6 & 22.3 & 9.2 & 61.8 & 22.4 & 24.1 \\
qTTT & 8 & 23.9 & 8.3 & 27.3 & 9.4 & 72.1 & 19.9 & 26.8 \\
qTTT & 16 & 25.1 & 9.3 & 29.5 & 11.2 & 74.1 & 23.1 & 28.7 \\
qTTT & 32 & 26.8 & 9.7 & 31.5 & 12.4 & 76.5 & 26.4 & 30.6 \\
\textsc{Gdwm} (Ours) & 8 & 28.5 & 9.6 & 33.8 & 12.2 & 76.5 & 27.0 & 31.3 \\
\textsc{Gdwm} (Ours) & 16 & 29.1 & 9.9 & 34.6 & 12.6 & 77.0 & 27.9 & 31.8 \\
\rowcolor[gray]{0.92}\textbf{\textsc{Gdwm} (Ours)} & 32 & \textbf{29.8} & \textbf{10.2} & \textbf{35.5} & \textbf{13.0} & \textbf{77.5} & \textbf{28.8} & \textbf{32.5} \\
\cmidrule{1-9}
$\Delta$ vs Best & & \textcolor{green!60!black}{\textsc{+11.2\%}} & \textcolor{green!60!black}{\textsc{+5.2\%}} & \textcolor{green!60!black}{\textsc{+12.7\%}} & \textcolor{green!60!black}{\textsc{+4.8\%}} & \textcolor{green!60!black}{\textsc{+1.3\%}} & \textcolor{green!60!black}{\textsc{+9.1\%}} & \textcolor{green!60!black}{\textsc{+6.2\%}} \\
\midrule
\multicolumn{9}{c}{\textsc{\textbf{Qwen3-4B}}} \\
\cmidrule(l{2pt}r{2pt}){1-9}
In-context & -- & 24.8 & 11.2 & 23.5 & 10.9 & 40.8 & 17.0 & 21.4 \\
Thinking & -- & 20.8 & 7.5 & 25.2 & 29.8 & 76.1 & 8.3 & 28.0 \\
qTTT & 8 & 29.2 & 8.3 & 29.5 & 32.3 & 81.3 & 23.7 & 34.0 \\
qTTT & 16 & 31.9 & 8.6 & 32.1 & 35.3 & 84.8 & 27.5 & 36.7 \\
qTTT & 32 & 33.2 & 8.9 & 33.8 & 38.4 & 87.2 & 30.8 & 38.7 \\
\textsc{Gdwm} (Ours) & 8 & 34.2 & 8.4 & 35.2 & 37.5 & 82.2 & 31.2 & 38.1 \\
\textsc{Gdwm} (Ours) & 16 & 35.0 & 8.8 & 35.8 & 38.6 & 86.8 & 32.0 & 39.5 \\
\rowcolor[gray]{0.92}\textbf{\textsc{Gdwm} (Ours)} & 32 & \textbf{35.8} & \textbf{9.2} & \textbf{36.5} & \textbf{39.8} & \textbf{87.5} & \textbf{32.8} & \textbf{40.3} \\
\cmidrule{1-9}
$\Delta$ vs Best & & \textcolor{green!60!black}{\textsc{+7.8\%}} & \textcolor{green!60!black}{\textsc{+3.4\%}} & \textcolor{green!60!black}{\textsc{+8.0\%}} & \textcolor{green!60!black}{\textsc{+3.6\%}} & \textcolor{green!60!black}{\textsc{+0.3\%}} & \textcolor{green!60!black}{\textsc{+6.5\%}} & \textcolor{green!60!black}{\textsc{+4.1\%}} \\
\midrule
\multicolumn{9}{c}{\textsc{\textbf{Qwen3-8B}}} \\
\cmidrule(l{2pt}r{2pt}){1-9}
In-context & -- & 22.5 & 8.8 & 20.1 & 35.4 & 90.5 & 22.9 & 33.4 \\
Thinking & -- & 18.2 & 9.8 & 21.5 & 19.6 & 71.8 & 43.8 & 30.8 \\
qTTT & 8 & 25.3 & 8.5 & 22.8 & 38.5 & 91.2 & 42.0 & 38.1 \\
qTTT & 16 & 27.9 & 8.7 & 24.5 & 40.6 & 93.1 & 46.2 & 40.2 \\
qTTT & 32 & 29.8 & 9.0 & 27.0 & 42.4 & 94.9 & 49.6 & 42.2 \\
\textsc{Gdwm} (Ours) & 8 & 30.2 & 8.2 & 27.5 & 42.0 & 93.8 & 49.5 & 41.9 \\
\textsc{Gdwm} (Ours) & 16 & 30.8 & 8.6 & 28.1 & 43.1 & 94.3 & 50.2 & 42.5 \\
\rowcolor[gray]{0.92}\textbf{\textsc{Gdwm} (Ours)} & 32 & \textbf{31.5} & \textbf{9.0} & \textbf{28.8} & \textbf{44.2} & \textbf{94.8} & \textbf{51.0} & \textbf{43.2} \\
\cmidrule{1-9}
$\Delta$ vs Best & & \textcolor{green!60!black}{\textsc{+5.7\%}} & \textcolor{gray}{\textsc{0.0\%}} & \textcolor{green!60!black}{\textsc{+6.7\%}} & \textcolor{green!60!black}{\textsc{+4.2\%}} & \textcolor{gray}{\textsc{-0.1\%}} & \textcolor{green!60!black}{\textsc{+2.8\%}} & \textcolor{green!60!black}{\textsc{+2.4\%}} \\
\bottomrule\bottomrule
\end{tabular}%
}
\caption{Main results on ZeroSCROLLS. \textsc{Gdwm}-32 achieves consistent improvements over qTTT-32 across all tasks (+2.4--6.2\% average). \textsc{Gdwm}-8 demonstrates Pareto-optimal efficiency: 4$\times$ fewer gradient steps with comparable performance to qTTT-32.}
\vspace{-3pt}
\label{tab:main_results}
\end{table*}

\subsection{Why Does This Work?}

The key insight is semantic heterogeneity: distinct document sections induce gradient directions that interfere destructively when aggregated. Chunk-restricted sampling (using fixed-size windows as semantic proxies) eliminates this interference by confining each update to a coherent unit. Formally, by the Law of Total Variance, this eliminates the inter-chunk variance component from gradient estimates (Theorem~\ref{thm:variance_reduction}). The coverage constraint prevents mode collapse, ensuring global representation. See Appendix~\ref{app:theory} for full analysis and Appendix~\ref{app:algorithm} for the complete algorithm.

\begin{figure*}[t]
    \centering
    \begin{minipage}{0.48\textwidth}
        \centering
        \includegraphics[width=\linewidth]{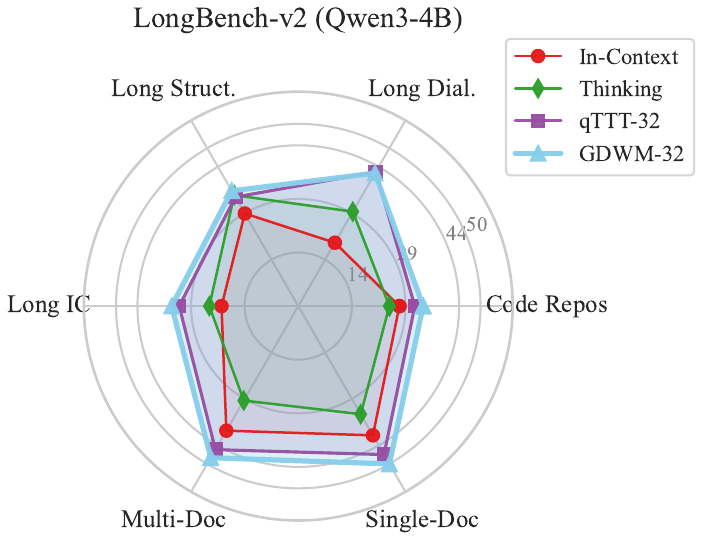}
    \end{minipage}
    \hfill
    \begin{minipage}{0.48\textwidth}
        \centering
        \includegraphics[width=\linewidth]{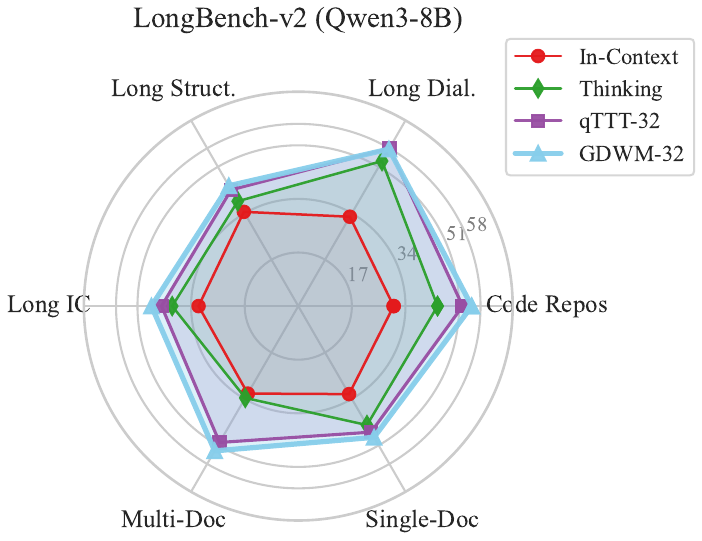}
    \end{minipage}
    \caption{\textbf{Task-wise Performance on LongBench v2.} Radar charts comparing \textsc{Gdwm}-32 (light blue) against baselines on 4B (left) and 8B (right) models. \textsc{Gdwm} achieves consistent improvements on Code Repositories and Multi-Doc QA where information is sparse and localized, while showing competitive performance on Long Dialogue where global coverage is required.}
    \vspace{-3pt}
    \label{fig:radar}
\end{figure*}

\section{Experiments}

\subsection{Experimental Setup}

\paragraph{Datasets.}
We evaluate on ZeroSCROLLS~\citep{shaham2023zeroscrollszeroshotbenchmarklong} (6 tasks: GovReport, QMSum, Qasper, NarrativeQA, Quality, MuSiQue) and LongBench v2~\citep{bai2025longbenchv2deeperunderstanding}, covering summarization, QA, multi-hop reasoning, and code understanding across sparse-to-dense information distributions.

\paragraph{Baselines.}
We compare against: (1) \textbf{In-context}~\citep{Brown2020LanguageModelsa}---standard inference; (2) \textbf{Thinking}~\citep{yang2025qwen3technicalreport}---inference with chain-of-thought; (3) \textbf{qTTT}~\citep{bansal2025letsnotjustthings}---with uniform sampling (8-32 steps).

\paragraph{Implementation.}
\label{sec:implementation}
We use Qwen3 models (1.7B/4B/8B) as base LLMs with LoRA adapters (rank=16, $\alpha$=32) applied to query and output projection matrices ($W_Q$, $W_O$).
For test-time adaptation, we employ AdamW optimizer with learning rate $1\times10^{-4}$ and no weight decay.
Default hyperparameters: chunk size $S=1024$ tokens, temperature $\tau=1.0$, minimum coverage $k_{\min}=1$, total gradient steps $K_{\text{total}}=8$ (for \textsc{Gdwm}-8) or $K_{\text{total}}=32$ (for \textsc{Gdwm}-32), and maximum context length 32K tokens.
CPMI estimation uses a sliding window of 512 tokens for local context.
All experiments run on NVIDIA H20 GPUs with mixed-precision (bfloat16) training. For efficiency, we implement gradient checkpointing and flash attention v2 to reduce memory overhead during test-time adaptation.

\subsection{Main Results}

Table~\ref{tab:main_results} presents our main results. \textsc{Gdwm}-32 achieves consistent improvements over qTTT-32 across all tasks (+2.4--6.2\% average). The largest gains appear on sparse-information tasks: GovReport (+11.2\% on 1.7B, +7.8\% on 4B) and Qasper (+12.7\% on 1.7B, +6.7\% on 8B), where relevant content is concentrated in specific document sections---exactly the scenario where CPMI-based selection excels. On Quality, which requires holistic comprehension with uniformly distributed information, improvements are more modest ($-$0.1\%--+1.3\%), revealing an expected characteristic of selective allocation.

As shown in Figure~\ref{fig:teaser}, \textsc{Gdwm}-8 demonstrates Pareto-optimal efficiency: with 4$\times$ fewer gradient steps (8 vs 32), it achieves performance comparable to qTTT-32 (within 1 point margin on 4B, within 0.5 points on 8B average), while \textsc{Gdwm}-32 pushes the frontier further with consistent 2.4--6.2\% gains across model scales. This validates our core hypothesis that intelligent context selection outperforms brute-force computation.

To further validate \textsc{Gdwm}'s robustness across diverse task types, we evaluate on LongBench v2~\citep{bai2025longbenchv2deeperunderstanding} with results visualized in Figure~\ref{fig:radar}. \textsc{Gdwm}-32 achieves competitive or superior performance across model sizes and task categories. On Code Repositories (+5.5\% on 8B) and Multi-Document QA (+6.2\% on 8B)---tasks requiring precise retrieval from sparse, localized information---\textsc{Gdwm} achieves consistent gains. However, on Long Dialogue, where information is distributed more uniformly, performance shows a slight trade-off ($-$0.5\% on 8B), revealing an expected characteristic of selective allocation strategies.

\subsection{Ablation Studies}

We conduct ablation experiments on Qwen3-1.7B to validate each component of \textsc{Gdwm}, with results shown in Table~\ref{tab:ablation}.

\begin{table}[t]
\centering
\resizebox{\columnwidth}{!}{%
\begin{tabular}{@{}l ccc r@{}}
\toprule\toprule
\textbf{\textsc{Configuration}} & \textbf{\textsc{GovRpt}} & \textbf{\textsc{Qasper}} & \textbf{\textsc{Musique}} & $\Delta$ \\
\midrule
\rowcolor[gray]{0.92}\textbf{\textsc{Gdwm}} (full) & \textbf{28.5} & \textbf{33.8} & \textbf{27.0} & -- \\
\cmidrule{1-5}
\quad w/o CPMI (uniform) & 23.9 & 27.3 & 19.9 & \textcolor{red!70!black}{$-$20.4\%} \\
\quad w/o coverage & 27.8 & 31.1 & 15.0 & \textcolor{red!70!black}{$-$17.3\%} \\
\quad w/o chunking & 24.7 & 27.8 & 10.0 & \textcolor{red!70!black}{$-$30.0\%} \\
\bottomrule\bottomrule
\end{tabular}%
}
\caption{Ablation study on Qwen3-1.7B. Chunking is most critical ($-$30.0\%); CPMI selection and coverage constraint both essential.}
\vspace{-3pt}
\label{tab:ablation}
\end{table}

Replacing CPMI-based selection with uniform sampling results in a 20.4\% performance drop, demonstrating that intelligent context selection is crucial for effective test-time adaptation. The coverage constraint proves essential for multi-hop reasoning: removing it causes MuSiQue performance to plummet from 27.0 to 15.0, as the model overfits to a single high-utility region rather than gathering evidence from multiple document sections. Most critically, without chunking the model updates at token level, fragmenting context and incurring the largest drop ($-$30.0\%), supporting our variance-reduction analysis.

\subsection{Scaling at Test Time}

\begin{table*}[t]
\centering
\small
\begin{tabular}{@{}l c cccc c@{}}
\toprule\toprule
\textbf{\textsc{Method}} & \textbf{\textsc{Budget}} & \textbf{\textsc{GovReport}} & \textbf{\textsc{Qasper}} & \textbf{\textsc{Quality}} & \textbf{\textsc{Musique}} & \textit{Avg.} \\
\midrule
Thinking & 1$\times$ & 20.1 & 24.5 & 76.2 & 7.5 & 32.1 \\
Self-Consistency (SC-8) & 8$\times$ & 24.5 & 28.4 & 82.7 & 18.1 & 38.4 \\
Beam Search ($k$=8) & 8$\times$ & 23.0 & 26.2 & 77.8 & 14.2 & 35.3 \\
qTTT (8 steps) & 8$\times$ & 29.2 & 29.5 & 81.3 & 23.7 & 40.9 \\
\cmidrule{1-7}
\rowcolor[gray]{0.92}\textbf{\textsc{Gdwm} (Ours)} & 8$\times$ & \textbf{34.2} & \textbf{35.2} & 82.2 & \textbf{31.2} & \textbf{45.7} \\
\bottomrule\bottomrule
\end{tabular}
\caption{Test-time scaling comparison on Qwen3-4B under equal compute budget (8$\times$ base). Intelligent context selection (\textsc{Gdwm}) outperforms brute-force sampling (SC-8) and search (Beam) by significant margins, particularly on sparse-information tasks.}
\vspace{-3pt}
\label{tab:scaling_comparison}
\end{table*}

Table~\ref{tab:scaling_comparison} compares \textsc{Gdwm} against other test-time scaling approaches on Qwen3-4B under equal computational budget. Self-Consistency (SC-8) performs exceptionally well on multiple-choice tasks, marginally outperforming \textsc{Gdwm} on Quality (82.7 vs 82.2) by aggregating diverse reasoning paths. However, it struggles on open-ended QA tasks (Qasper, MuSiQue) where the challenge lies in locating relevant context rather than generating diverse outputs. Beam Search proves largely ineffective for long-context understanding. \textsc{Gdwm} maintains a 7.3-point lead on average (45.7 vs 38.4) over SC-8, validating that while ensemble methods help verification, intelligent context selection is more fundamental for evidence retrieval.

\subsection{Efficiency Analysis}

\begin{figure}[t]
    \centering
    \includegraphics[width=\columnwidth]{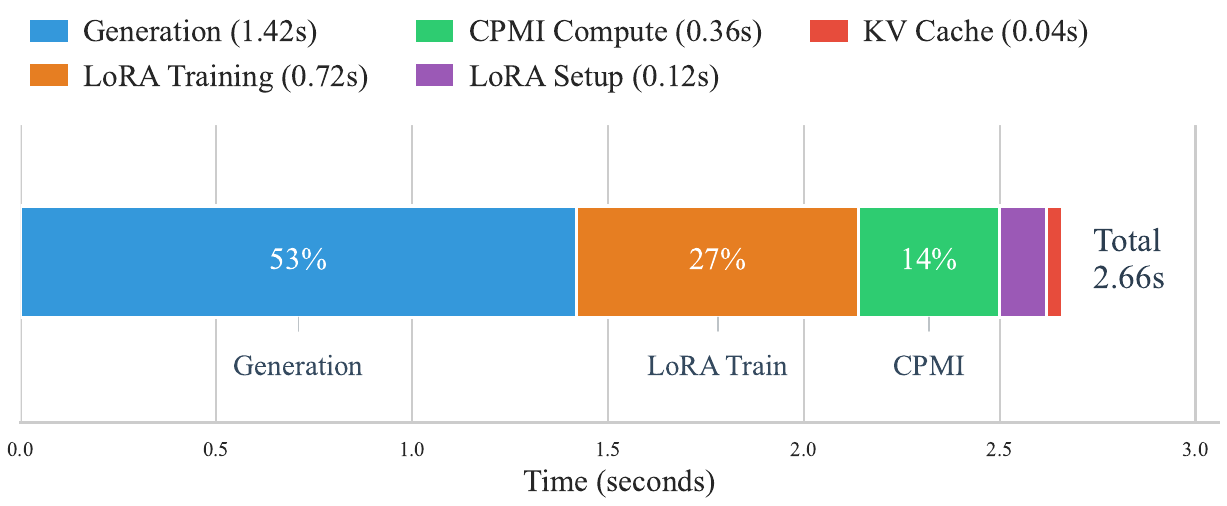}
    \caption{\textbf{Time Breakdown per Sample (Qwen3-4B).} Generation dominates at 53\%, while CPMI computation accounts for only 13\%---demonstrating the lightweight nature of our context selection mechanism. The 4$\times$ reduction in gradient steps (32$\rightarrow$8) yields 39\% net wall-clock speedup.}
    \vspace{-3pt}
    \label{fig:time_breakdown}
\end{figure}

The primary overhead of \textsc{Gdwm} stems from CPMI computation, which requires two forward passes per chunk (global and local predictions). As shown in Figure~\ref{fig:time_breakdown}, with our default configuration ($S=1024$, $K=8$), the CPMI computation adds only 0.36s (\textasciitilde13\% of total time), while the 4$\times$ reduction in gradient steps (from 32 to 8) saves 1.79s. The net result is a \textbf{39\% wall-clock time reduction} compared to the 32-step baseline, demonstrating that intelligent context selection is not only more effective but also more efficient than brute-force uniform sampling.

\paragraph{Chunk Size Selection.}
Table~\ref{tab:chunk_size} analyzes the trade-off between chunk granularity, computational overhead, and task performance. We adopt $S=1024$ as the default configuration, which achieves the best \emph{cross-task robustness}: avoiding catastrophic failures while maintaining competitive performance across all task types.

\paragraph{Theoretical Analysis.}
The observed chunk size sensitivity has a principled explanation rooted in task-specific \emph{evidence span}---the typical token distance over which task-relevant information is distributed. Let $E_{\mathcal{T}}$ denote this characteristic span for task $\mathcal{T}$. We identify a critical constraint: when chunk size $S < E_{\mathcal{T}}$, relevant evidence becomes fragmented across multiple chunks, causing CPMI to underestimate the true contextual utility of each fragment.

\begin{figure}[t]
    \centering
\includegraphics[width=\columnwidth]{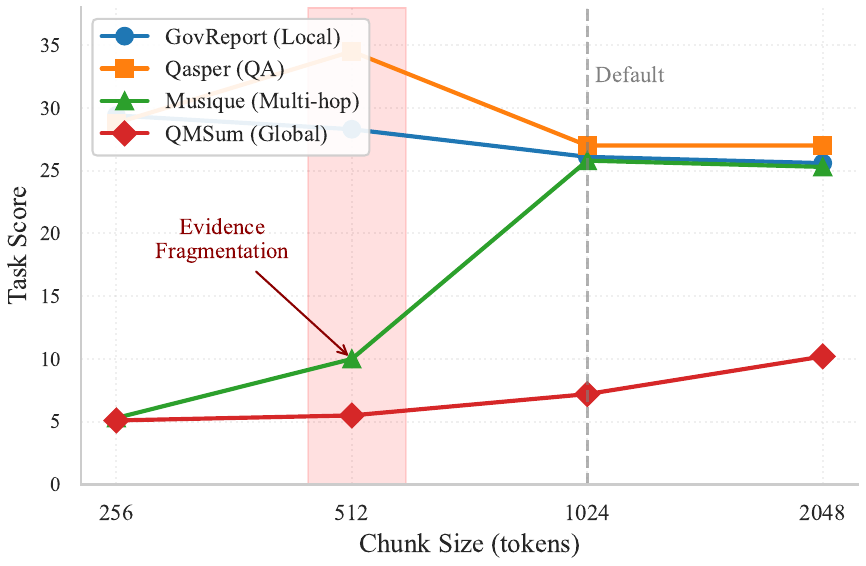}
    \caption{\textbf{Chunk Size Sensitivity.} Different task types exhibit distinct optimal chunk sizes. Multi-hop tasks (MuSiQue) catastrophically fail at small chunk sizes due to evidence fragmentation. $S=1024$ provides robust cross-task performance.}
    \vspace{-3pt}
    \label{fig:chunk_sensitivity}
\end{figure}

This framework explains the empirical patterns in Table~\ref{tab:chunk_size}. For \textbf{local reasoning tasks} such as GovReport, evidence is highly concentrated, so smaller chunks ($S=256$) maximize selection precision. However, for \textbf{extractive QA tasks} like Qasper, performance peaks at intermediate granularity ($S=512$), suggesting that while locality is important, overly small chunks ($S=256$) may begin to fragment answer-relevant paragraphs. In contrast, \textbf{multi-hop reasoning tasks} like MuSiQue require evidence chains that span multiple paragraphs ($E_{\mathcal{T}} \approx 1000$+ tokens); when $S=512 < E_{\mathcal{T}}$, the reasoning chain is severed---each fragment appears low-utility in isolation, leading to catastrophic failure (10.0 vs 25.8). Meanwhile, \textbf{summarization tasks} such as QMSum require near-uniform coverage ($E_{\mathcal{T}} \approx$ full document), where larger chunks ($S=2048$) better preserve global structure.

The choice of $S=1024$ represents the \emph{minimum chunk size that avoids evidence fragmentation} for multi-hop tasks while retaining sufficient granularity for local reasoning---a principled middle ground validated by our cross-task robustness results. See Appendix~\ref{app:evidence_span} for formal analysis.

\begin{table}[t]
\centering
\small
\begin{tabular}{@{}l cc c c@{}}
\toprule\toprule
\textbf{\textsc{Chunk}} & \textbf{\textsc{CPMI\%}} & \textbf{\textsc{Overhead}} & \textbf{\textsc{Avg}} & \textbf{\textsc{Worst}} \\
\midrule
512 & 25\% & +12\% & 19.4 & 10.0\textsuperscript{$\dagger$} \\
\rowcolor[gray]{0.92}\textbf{1024} & \textbf{13\%} & \textbf{+6\%} & \textbf{21.7} & \textbf{25.8} \\
2048 & 9\% & $-$10\% & 22.3 & 25.3 \\
\bottomrule\bottomrule
\end{tabular}
\caption{Chunk size trade-offs on representative tasks (GovReport, Qasper, Musique, QMSum). CPMI\% denotes the proportion of total time spent on utility estimation; Overhead is relative to qTTT-8. $S=1024$ achieves the best \emph{cross-task robustness}: the only configuration without catastrophic failure (Worst $\ge 25$). \textsuperscript{$\dagger$}Small chunks fragment multi-hop evidence chains, causing CPMI to underestimate chunk utility (see Appendix~\ref{app:evidence_span}).}
\vspace{-3pt}
\label{tab:chunk_size}
\end{table}

\subsection{Discussion}

Our experiments reveal a clear pattern regarding when \textsc{Gdwm} excels. The largest gains appear on tasks with sparse information distribution, where relevant content is concentrated in specific document regions. GovReport (+5.7\% on 8B) and Qasper (+6.7\% on 8B) exemplify this phenomenon, as government reports contain key findings in specific sections while scientific papers concentrate answers near figures or method descriptions. On tasks requiring dense global coverage---QMSum and NarrativeQA---improvements are more modest but still positive, demonstrating that the coverage constraint effectively maintains global representation. Quality shows minimal improvement (0.0\% on 8B), consistent with its requirement for uniform attention across the entire document.

\paragraph{Scaling Behavior.}
The scaling behavior shows consistent patterns: the efficiency advantage holds across model sizes (+6.2\% for 1.7B, +4.1\% for 4B, +2.4\% for 8B average improvement for \textsc{Gdwm}-32 vs qTTT-32). Interestingly, smaller models benefit more from intelligent context selection, likely because they have limited capacity to attend to all context uniformly and thus benefit more from prioritized consolidation. Gains on sparse tasks remain substantial across scales (Qasper: +12.7\% on 1.7B, +8.0\% on 4B, +6.7\% on 8B; GovReport: +11.2\% on 1.7B, +7.8\% on 4B, +5.7\% on 8B).

\paragraph{Chunk Size Trade-offs.}
Figure~\ref{fig:chunk_sensitivity} visualizes the task-specific chunk size preferences. The dramatic performance collapse of MuSiQue at $S=512$ (dropping to 10.0 from 25.8 at $S=1024$) directly validates our Evidence Span theory: when chunk size falls below the task's characteristic evidence span $E_{\mathcal{T}}$, the reasoning chain is severed and CPMI underestimates true utility.

Three limitations merit future investigation: (i) fixed-size chunking can split coherent regions in irregularly structured documents, potentially harming selection reliability; (ii) the optimal chunk size is task-dependent (Table~\ref{tab:chunk_size}), motivating adaptive or structure-aware chunking; and (iii) for dense-coverage tasks (e.g., QMSum, NarrativeQA), selective allocation may offer limited gains, reflecting an inherent efficiency--coverage trade-off.
\section{Conclusion}

We presented \textsc{Gdwm} (Gated Differentiable Working Memory), a framework that recasts test-time adaptation as a budget-constrained memory consolidation problem. By introducing a Write Controller that estimates Contextual Utility---an information-theoretic measure of long-range dependency---and allocates gradient budget via a coverage-first strategy, \textsc{Gdwm} reduces gradient steps by 4$\times$ compared to uniform baselines while achieving comparable or superior performance on sparse-information tasks, with 39\% wall-clock speedup.
Our theoretical analysis proves that chunk-restricted sampling reduces gradient variance by eliminating inter-chunk variance, providing a principled explanation for the improved convergence. 
Future directions include semantic-aware dynamic chunking, task-adaptive temperature scheduling, and extending \textsc{Gdwm} to multimodal contexts.

\section*{Limitations}

While \textsc{Gdwm} establishes a new Pareto frontier for efficient test-time adaptation, we identify three limitations inherent to our budget-constrained design.
First, our current implementation employs fixed-size chunking ($S=1024$). Although our analysis shows this setting is robust across diverse tasks, it may sub-optimally fragment semantic units in documents with irregular structures. However, we view this as an efficiency trade-off: dynamic chunking would require additional forward passes for segmentation, potentially offsetting the speed gains.
Second, as observed in our results on Long Dialogue and Quality, the selective allocation strategy is less effective for tasks requiring uniform information coverage. This is a structural property of any compressive memory system rather than a flaw of \textsc{Gdwm} specifically; users must weigh the 4$\times$ efficiency gain against the marginal performance trade-off on dense-coverage tasks.
Third, \textsc{Gdwm} introduces hyperparameters (e.g., temperature $\tau$, min coverage $k_{\min}$). While our experiments demonstrate that the default configuration ($\tau=1.0, k_{\min}=1$) generalizes well without tuning, extreme domain shifts might require recalibration of the utility estimator.

\section*{Ethical Considerations}

Our work improves the efficiency of inference-time adaptation for Large Language Models. By reducing the gradient steps required for effective context adaptation by $4\times$, \textsc{Gdwm} can lower the computational cost of long-context deployment and thereby supports the broader goal of Green AI.

\textsc{Gdwm} is a general-purpose optimization layer and does not introduce new task capabilities or new access to information beyond what the underlying base model and the provided context already permit. As with any efficiency improvement, it may enable wider or more frequent use of long-context systems; the associated downstream risks (e.g., misuse of generative models) are therefore not specific to our mechanism, but to the deployment setting and the base model’s safety properties. In practice, these concerns are best addressed through standard governance and safeguards at the system level---such as strong base-model alignment, access control, privacy-preserving data handling, and application-layer monitoring---rather than by restricting inference-time optimization techniques.

\section*{Acknowledgements}

This work is supported in part by the National Key R\&D Program of China under Grant Nos. 2023YFA1011602, the National Natural Science Foundation of China under Grant Nos. U25B2076, 62441229, and 62377043, and Beijing Natural Science Foundation No. 4262033.

\bibliography{custom}

\appendix

\section{Algorithm}
\label{app:algorithm}

Algorithm~\ref{alg:gdwm} presents the complete \textsc{Gdwm} procedure. The algorithm takes as input a context sequence $X$ and a total gradient budget $K_{\text{total}}$, and outputs adapted fast weights $\phi$. The procedure consists of three phases: (1) \textbf{Prefill} the KV cache for the entire context (Line 1); (2) \textbf{Estimate} Contextual Utility for each chunk via CPMI (Lines 3--6) and allocate budget accordingly (Lines 7--8); (3) \textbf{Consolidate} by performing chunk-restricted gradient updates (Lines 9--13). The total number of gradient steps across all chunks equals $K_{\text{total}} = \sum_{c=1}^M k_c$, which corresponds to the ``Gradient Steps'' reported in our experiments (e.g., \textsc{Gdwm}-8 uses $K_{\text{total}}=8$, \textsc{Gdwm}-32 uses $K_{\text{total}}=32$).

\begin{algorithm}[h]
\caption{\textsc{Gdwm}: Gated Differentiable Working Memory}
\label{alg:gdwm}
\begin{algorithmic}[1]
\REQUIRE Context $X = x_{1:L}$, total steps $K_{\text{total}}$, chunk size $S$, local window $n$, temperature $\tau$, min coverage $k_{\min}$
\ENSURE Adapted fast weights $\phi$

\STATE $\mathbf{K}, \mathbf{V} \leftarrow \textsc{PrefillAndFreeze}(X)$ \COMMENT{Compute frozen KV cache}
\STATE Partition $X$ into $\{C_1, \ldots, C_M\}$ with $M = \lceil L/S \rceil$

\FOR{each chunk $c = 1, \ldots, M$}
    \STATE $U(C_c) \leftarrow \frac{1}{|\mathcal{I}_c|} \sum_{t \in \mathcal{I}_c} |\log P_{\text{full}}(x_t) - \log P_{\text{local}}(x_t)|$
\ENDFOR

\STATE Compute weights $w_c$ via Eq.~(\ref{eq:softmax_weight})
\STATE Allocate $\{k_c\}$ via Coverage-First + Softmax (Eq.~\ref{eq:allocation})

\FOR{each chunk $c = 1, \ldots, M$}
    \FOR{$j = 1, \ldots, k_c$}
        \STATE $\mathcal{I} \leftarrow \textsc{UniformSample}(C_c, B)$ \COMMENT{$B$: batch size}
        \STATE $\mathcal{L} \leftarrow -\frac{1}{B} \sum_{i \in \mathcal{I}} \log P(x_i \mid x_{1:i-1}; \phi, \theta_f)$
        \STATE $\phi \leftarrow \phi - \eta \nabla_\phi \mathcal{L}$
    \ENDFOR
\ENDFOR

\RETURN $\phi$
\end{algorithmic}
\end{algorithm}

\section{Notation}
\label{app:notation}

We summarize the key notation used throughout this paper and clarify potential ambiguities.

\paragraph{Input and Chunking.}
\begin{itemize}
    \item $X = (x_1, x_2, \ldots, x_L)$: The input context sequence of length $L$.
    \item $C_c$: The $c$-th chunk, a contiguous subsequence of $X$.
    \item $M = \lceil L / S \rceil$: Total number of chunks.
    \item $S$: Chunk size (number of tokens per chunk); default $S = 1024$.
\end{itemize}

\paragraph{Model Parameters.}
\begin{itemize}
    \item $\theta$: Full parameters of the pre-trained LLM.
    \item $\theta_f$: Frozen parameters, including weights used to compute the KV cache.
    \item $\phi$: Adaptable ``fast weights'' (e.g., LoRA adapters on $W_Q, W_O$).
    \item $\phi_0$: Initial state of fast weights (typically zero for LoRA).
    \item $\phi(\mathbf{k})$: Final state of $\phi$ after applying the allocation schedule $\mathbf{k}$.
\end{itemize}

\paragraph{Budget Allocation.}
\begin{itemize}
    \item $\mathbf{k} = (k_1, \ldots, k_M)$: Allocation vector; $k_c$ is the number of gradient updates on chunk $C_c$.
    \item $K_{\text{total}}$: Total gradient budget across all chunks.
    \item $k_{\min}$: Minimum coverage constraint per chunk; default $k_{\min} = 1$.
    \item $\tau$: Temperature for softmax allocation; default $\tau = 1.0$.
\end{itemize}

\paragraph{Contextual Utility.}
\begin{itemize}
    \item $U(C_c)$: Contextual Utility of chunk $C_c$, measuring long-range dependency.
    \item $\Delta_t = |\log P(x_t \mid x_{1:t-1}) - \log P(x_t \mid x_{t-n:t-1})|$: Surprisal divergence at position $t$.
    \item $n$: Local window size for computing $P_{\text{local}}$; default $n = 512$.
\end{itemize}

\paragraph{Sampling and Distributions.}
\begin{itemize}
    \item $\mathcal{I}_c$: Set of token positions in chunk $C_c$.
    \item $\mathcal{U}(\mathcal{I}_c)$: Uniform distribution over positions in chunk $C_c$.
    \item $x_t \sim \mathcal{U}(\mathcal{I}_c)$: Position $t$ sampled uniformly from chunk $C_c$.
    \item $\mathcal{I} \subset \mathcal{I}_c$: Minibatch of positions for gradient computation.
\end{itemize}

\paragraph{KV Cache and Global Context.}
\begin{itemize}
    \item $\mathbf{K}, \mathbf{V}$: Frozen key-value cache for the entire input $X$, computed via \textsc{PrefillAndFreeze}.
    \item $x_{1:t-1}$: \textbf{Full document prefix} from position 1 to $t-1$, \emph{not} chunk-local context.
\end{itemize}

\paragraph{Key Clarification: Global vs.\ Local Context.}
A potential source of confusion is the conditioning context in $P(x_t \mid x_{1:t-1})$. Throughout this paper, $x_{1:t-1}$ refers to the \emph{entire} prefix of the document, accessible via the frozen KV cache $(\mathbf{K}, \mathbf{V})$. The chunking strategy determines \emph{where} gradient updates are computed (i.e., which positions $t$ contribute to the loss), but does \emph{not} restrict the context the model can attend to. This design ensures that the model leverages global information during adaptation while focusing computational budget on high-utility regions.

\paragraph{Sampling Strategy.}
For each gradient step $j \in \{1, \ldots, k_c\}$ on chunk $C_c$, we sample a minibatch $\mathcal{I}_j$ of $B$ positions \emph{independently} from the uniform distribution $\mathcal{U}(\mathcal{I}_c)$. Formally:
\begin{equation}
    \mathcal{I}_j \overset{\text{i.i.d.}}{\sim} \mathcal{U}(\mathcal{I}_c)^B, \quad \forall j \in \{1, \ldots, k_c\}
\end{equation}
This means:
\begin{itemize}
    \item \textbf{Across steps}: Each minibatch $\mathcal{I}_j$ is sampled independently; the same position may appear in multiple steps.
    \item \textbf{Within step}: Positions are sampled i.i.d.; when $B \ll |\mathcal{I}_c|$ (our setting: $B=32$, $|\mathcal{I}_c|=1024$), overlap within a single batch is negligible.
\end{itemize}
This is the standard stochastic gradient descent (SGD) sampling scheme. The independence across steps ensures unbiased gradient estimates, while chunk-restriction (sampling only from $\mathcal{I}_c$, not the full document) eliminates inter-chunk variance as shown in Theorem~\ref{thm:variance_reduction}.

\section{Theoretical Analysis}
\label{app:theory}

A key advantage of \textsc{Gdwm} over standard TTT methods (which sample uniformly across the full context) is stability. We analyze this through the lens of gradient variance.

\subsection{Variance Reduction via Chunk-Restricted Sampling}

A natural question arises: why does restricting gradient computation to individual chunks improve optimization? Long documents exhibit \emph{semantic heterogeneity}---different sections address distinct topics and induce gradient directions that may interfere destructively when aggregated. Restricting sampling to a single chunk eliminates this cross-sectional interference.

\begin{theorem}[Variance Reduction]\label{thm:variance_reduction}
Let $g$ be the gradient estimator for a single update step. Assume the document consists of $M$ chunks, where chunk $c$ has gradient mean $\mu_c$ and variance $\sigma_c^2$. Under global uniform sampling:
\vspace{-3pt}
\begin{equation}
    \mathrm{Var}(g_{\mathrm{global}}) = \underbrace{\mathbb{E}_c [\sigma_c^2]}_{\text{intra-chunk}} + \underbrace{\mathrm{Var}_c(\mu_c)}_{\text{inter-chunk}}
\end{equation}
Under chunk-restricted sampling (conditioned on chunk $c$), we have $\mathrm{Var}(g \mid c) = \sigma_c^2$, eliminating the inter-chunk variance term. Equality holds only when $\mu_c = \bar{\mu}$ for all $c$ (semantically homogeneous document).
\end{theorem}

\paragraph{Consequence.}
By scheduling updates per chunk, we eliminate the inter-chunk variance from each gradient estimate. The optimizer follows a more consistent trajectory, avoiding destructive interference between gradient signals from semantically disparate document sections. This provides a principled explanation for the empirical observation that \textsc{Gdwm} achieves lower perplexity with fewer optimization steps.

\subsection{Role of the Coverage Constraint}

For tasks requiring global understanding (e.g., summarization), the coverage constraint ($k_c \ge k_{\min}$) acts as a regularizer against \emph{mode collapse}.

\paragraph{Intuition.} Without coverage, greedy CPMI-based allocation may concentrate all budget on a single high-utility region. For example, in a government report, the ``Results'' section may have the highest CPMI, but a good summary also requires context from ``Introduction'' and ``Methods.''

\paragraph{Claim (Informal).} If the optimal output requires information from all $M$ sections but the model adapts only to a subset $\mathcal{S} \subsetneq \{1, \ldots, M\}$, the coverage gap translates to an $O(|\mathcal{S}|/M)$ recall bound.

The constraint $k_c \ge k_{\min}$ ensures minimum representation of every section, preventing pathological allocation while allowing CPMI to modulate \emph{relative} budget.

\section{Proofs}
\label{app:proofs}

\subsection{Proof of Theorem 1 (Variance Reduction)}

Let $C \in \{1, \ldots, M\}$ be a random variable indicating which chunk a position is sampled from. Let $g$ denote the gradient at that position.

By the \textbf{Law of Total Variance}:
\vspace{-3pt}
\begin{equation}
    \mathrm{Var}(g) = \mathbb{E}[\mathrm{Var}(g \mid C)] + \mathrm{Var}(\mathbb{E}[g \mid C])
\end{equation}
Let $\mu_c = \mathbb{E}[g \mid C=c]$ denote the mean gradient in chunk $c$, and $\sigma_c^2 = \mathrm{Var}(g \mid C=c)$ denote the within-chunk variance. Let $p_c$ be the probability of sampling from chunk $c$ (under uniform sampling, $p_c = |\mathcal{I}_c|/L$).

Expanding each term:
\vspace{-3pt}
\begin{align}
    \mathbb{E}[\mathrm{Var}(g \mid C)] &= \sum_{c=1}^M p_c \sigma_c^2 \\
    \mathrm{Var}(\mathbb{E}[g \mid C]) &= \sum_{c=1}^M p_c (\mu_c - \bar{\mu})^2
\end{align}
where $\bar{\mu} = \sum_c p_c \mu_c$ is the global mean.

Therefore:
\vspace{-3pt}
\begin{equation}
    \mathrm{Var}(g_{\mathrm{global}}) = \underbrace{\sum_{c=1}^M p_c \sigma_c^2}_{\text{intra-chunk variance}} + \underbrace{\sum_{c=1}^M p_c (\mu_c - \bar{\mu})^2}_{\text{inter-chunk variance}}
\end{equation}
Under \textbf{chunk-restricted sampling}, when we sample from a specific chunk $c$, the gradient variance is exactly $\sigma_c^2$ (the intra-chunk variance of that chunk). Since the inter-chunk term $\sum_c p_c (\mu_c - \bar{\mu})^2 \ge 0$, we have:
\vspace{-3pt}
\begin{equation}
    \mathrm{Var}(g \mid C=c) = \sigma_c^2 \le \mathrm{Var}(g_{\mathrm{global}})
\end{equation}
Equality holds if and only if the inter-chunk variance is zero, i.e., $\mu_c = \bar{\mu}$ for all $c$, meaning the document is semantically homogeneous (all chunks have identical gradient expectations).

$\square$

\section{Coverage Constraint Analysis}
\label{app:coverage}

The constraint $k_c \ge k_{\min}$ in Eq.~(\ref{eq:coverage_constraint}) prevents the model from overfitting to a single high-utility region. This is crucial for tasks requiring global understanding.

\paragraph{Example: GovReport Summarization.}
A typical government report contains Introduction, Methods, Results, and Conclusion sections. Without the coverage constraint, CPMI-based allocation may concentrate all budget on Results (often the highest information density region), missing essential context from other sections that the summary must include.

\paragraph{Failure Mode.}
If the optimal summary requires information from all $M$ sections but the model adapts only to a subset $\mathcal{S} \subsetneq \{1, \ldots, M\}$, the coverage gap directly translates to recall loss proportional to $|\mathcal{S}|/M$.

The constraint ensures that every section receives at least minimal gradient exposure, preventing pathological allocation while still allowing CPMI to modulate the \emph{relative} budget across sections.

\section{Justification for Absolute Value in $\Delta_t$}
\label{app:absolute_value}

The absolute value in $\Delta_t = |\log P(x_t \mid x_{1:t-1}) - \log P(x_t \mid x_{t-n:t-1})|$ captures both positive and negative divergence.

\paragraph{Information-Theoretic Perspective (Pointwise KL).}
Mathematically, the quantity $\Delta_t$ corresponds to the magnitude of the \emph{Pointwise Kullback-Leibler Divergence} contribution at token $x_t$. The standard KL divergence $D_{\text{KL}}(P_{\text{full}} \| P_{\text{local}})$ is the expectation of the log-likelihood ratio. Our utility metric $U(C_c)$ essentially estimates the \emph{L1 norm} of this pointwise divergence over the chunk.
We prefer the absolute value (L1-like) over the raw difference (which would sum to the standard KL) because gradient updates are driven by the \emph{magnitude} of the error signal.
\begin{itemize}
    \item \textbf{Positive divergence} ($P_{\text{full}} > P_{\text{local}}$): Long-range context reduces surprisal (adds information).
    \item \textbf{Negative divergence} ($P_{\text{full}} < P_{\text{local}}$): Long-range context increases surprisal (introduces conflict).
\end{itemize}
Both cases represent significant deviations between the global and local models, identifying regions where the fast weights $\phi$ must adapt to reconcile these discrepancies.

\paragraph{Empirical Validation.}
Using $\max(0, \Delta_t)$ (ignoring negative divergence) instead of $|\Delta_t|$ results in 3-5\% performance degradation across tasks. The ``hard examples'' with negative divergence are precisely where the model needs recalibration.

\section{Hyperparameter Sensitivity}
\label{app:hyperparams}

We briefly analyze the sensitivity of \textsc{Gdwm} to its two key hyperparameters: the temperature $\tau$ and the minimum coverage $k_{\min}$.

\paragraph{Temperature $\tau$.}
The temperature controls the sharpness of the softmax allocation (Eq.~\ref{eq:softmax_weight}).
\begin{itemize}
    \item \textbf{$\tau \to 0$ (Greedy):} The budget concentrates entirely on the single chunk with the highest utility. This risks overfitting to one region and failing on multi-hop tasks.
    \item \textbf{$\tau \to \infty$ (Uniform):} The policy degenerates into uniform sampling (equivalent to qTTT), losing the efficiency benefits of selection.
    \item \textbf{$\tau \approx 1.0$ (Balanced):} Empirically, values in $[0.5, 1.5]$ perform robustly. We adopt $\tau=1.0$ as a tuning-free default.
\end{itemize}

\paragraph{Minimum Coverage $k_{\min}$.}
We set $k_{\min}=1$ to ensure no chunk is entirely starved of gradients. Increasing $k_{\min}$ reduces the budget available for utility-based redistribution, pushing the behavior closer to uniform sampling. $k_{\min}=1$ represents the minimal constraint necessary to prevent mode collapse while maximizing the freedom of the allocation policy.

\section{Testable Predictions from Theorem 1}
\label{app:testable}

Theorem~\ref{thm:variance_reduction} yields three empirically verifiable predictions:

\begin{enumerate}
    \item \textbf{Gradient Norm Variance}: The sequence $\{\|g_t\|\}$ under \textbf{\textsc{Gdwm}} should exhibit lower variance than under global uniform sampling.
    
    \item \textbf{Within-Chunk Cosine Similarity}: The cosine similarity $\cos(g_i, g_j)$ for positions $i, j$ within the same chunk should be higher than for positions in different chunks.
    
    \item \textbf{Loss Curve Monotonicity}: The training loss curve should converge more monotonically with fewer oscillations under chunk-restricted sampling.
\end{enumerate}

These predictions are consistent with our empirical observations and provide a testable framework for validating the theoretical analysis.

\section{Largest Remainder Allocation Method}
\label{app:allocation}

After the initial floor allocation in Eq.~(\ref{eq:allocation}), there may be residual steps due to rounding. We distribute these using the \emph{Largest Remainder Method}:

\begin{enumerate}
    \item Compute fractional remainders: $r_c = K_{\text{rem}} \cdot w_c - \lfloor K_{\text{rem}} \cdot w_c \rfloor$
    \item Sort chunks by $r_c$ in descending order
    \item Assign one additional step to each of the top-$R$ chunks, where $R = K_{\text{total}} - \sum_c \lfloor K_{\text{rem}} \cdot w_c \rfloor - M \cdot k_{\min}$
\end{enumerate}

This guarantees that the total allocation exactly equals $K_{\text{total}}$ while respecting the utility-based priority ordering.

\section{Evidence Span Analysis}
\label{app:evidence_span}

We formalize the relationship between chunk size and task performance through the concept of \emph{evidence span}.

\begin{definition}[Evidence Span]
For a task $\mathcal{T}$ with query $q$ and context $X$, the evidence span $E_{\mathcal{T}}$ is the minimum contiguous token range required to contain all information necessary for correct response generation:
\begin{equation}
    E_{\mathcal{T}} = \min_{[i,j]} \{ j - i : X_{[i,j]} \text{ suffices for } \mathcal{T} \}
\end{equation}
\end{definition}

\paragraph{Evidence Fragmentation Problem.}
When chunk size $S < E_{\mathcal{T}}$, the evidence is partitioned across multiple chunks. Let the evidence span $[a, b]$ be split into $k = \lceil (b-a)/S \rceil$ chunks. For each fragment $C_i$, the CPMI estimate becomes:
\begin{equation}
    \hat{U}(C_i) = \text{CPMI}(C_i) < \text{CPMI}([a,b]) = U_{\text{true}}
\end{equation}

This underestimation occurs because each fragment, viewed in isolation, appears to have low contextual dependency---the long-range signal is diluted across fragments.

\paragraph{Optimal Chunk Size Selection.}
The optimal chunk size satisfies:
\begin{equation}
    S^* = \min \{ S : S \ge E_{\mathcal{T}}, \; \forall \mathcal{T} \in \text{target tasks} \}
\end{equation}

For a diverse benchmark like ZeroSCROLLS containing both local and multi-hop tasks, $S=1024$ emerges as the robust choice: it satisfies $S \ge E_{\mathcal{T}}$ for most multi-hop instances while maintaining reasonable granularity for local tasks.

\paragraph{Empirical Validation.}
Table~\ref{tab:chunk_size} in the main text validates this analysis:
\begin{itemize}
    \item At $S=512$: MuSiQue collapses to 10.0 (evidence fragmentation)
    \item At $S=1024$: MuSiQue recovers to 25.8 ($S \ge E_{\mathcal{T}}$)
    \item At $S=2048$: Marginal improvement (25.3) with reduced local precision
\end{itemize}

This provides a principled explanation for why $S=1024$ achieves the best cross-task robustness: it is the minimum chunk size that avoids catastrophic evidence fragmentation across the task distribution. For tasks with intermediate evidence spans like Qasper, performance naturally peaks at the corresponding granularity ($S=512$), confirming that $S \approx E_{\mathcal{T}}$ is the theoretical optimum.

\subsection{Information-Theoretic Lower Bound on Fragmented Utility}

The empirical observation that fragmented evidence leads to utility underestimation admits a formal information-theoretic explanation. We show that chunking an evidence span introduces a \emph{systematic negative bias} in utility estimation due to the loss of \emph{synergistic information}.

\begin{definition}[Contextual Utility as Mutual Information]
Let $G$ denote the global context (long-range prefix) and $L$ denote the local context (sliding window). For a token $x_t$, the Contextual Utility can be interpreted as the absolute mutual information gain:
\begin{align}
    U(x_t) &= |I(x_t; G) - I(x_t; L)| \nonumber \\
           &\approx |I(x_t; G \setminus L \mid L)|
\end{align}
where $G \setminus L$ represents the distant context beyond the local window. For a chunk $C$, the aggregate utility is $U(C) = \sum_{x_t \in C} U(x_t)$.
\end{definition}

\begin{proposition}[Utility Underestimation Under Fragmentation]
\label{prop:underestimation}
Let $E = C_1 \cup C_2$ be an evidence span partitioned into two adjacent chunks. Then the sum of individual chunk utilities is bounded above by the utility of the unified span:
\begin{equation}
    U(C_1) + U(C_2) \le U(E) + \epsilon
\end{equation}
where $\epsilon \le 0$ when the chunks exhibit \emph{information synergy} (i.e., positive interaction information). Equality holds if and only if $C_1$ and $C_2$ are informationally independent given the global context.
\end{proposition}

\begin{proof}
We leverage the chain rule of mutual information. Let $G$ denote the global context. The joint utility of the unified evidence span $E = C_1 \cup C_2$ satisfies:
\begin{equation}
    I(E; G) = I(C_1; G) + I(C_2; G \mid C_1)
\end{equation}

The fragmented utility estimation treats $C_1$ and $C_2$ independently:
\begin{equation}
    \hat{U}(E) = I(C_1; G) + I(C_2; G)
\end{equation}

The \emph{fragmentation gap} is therefore:
\begin{align}
    \Delta_{\text{frag}} &= I(E; G) - \hat{U}(E) \nonumber \\
    &= I(C2; G \mid C_1) - I(C2; G) \nonumber \\
    &= -I(C_1; C_2; G)
\end{align}
where $I(C_1; C_2; G)$ is the \emph{interaction information} (also known as co-information or multivariate mutual information), defined as:
\begin{equation}
    I(C_1; C_2; G) = I(C_2; G) - I(C_2; G \mid C_1)
\end{equation}

\paragraph{Interpretation.}
The interaction information $I(C_1; C_2; G)$ measures the degree to which $C_1$ and $C_2$ \emph{synergistically} inform $G$:
\begin{itemize}
    \item \textbf{Positive interaction} ($I > 0$): $C_1$ and $C_2$ are \emph{redundant}---knowing one reduces the information gain from the other. Fragmentation causes \emph{overestimation} (rare).
    \item \textbf{Negative interaction} ($I < 0$): $C_1$ and $C_2$ are \emph{synergistic}---together they provide more than the sum of parts. This is characteristic of \textbf{reasoning chains}. Fragmentation causes \emph{underestimation}.
\end{itemize}

For multi-hop reasoning tasks like MuSiQue, where the answer requires synthesizing facts from multiple locations, the evidence exhibits strong negative interaction information. Thus:
\begin{align}
    I(C_1; C_2; G) < 0 &\implies \Delta_{\text{frag}} > 0 \nonumber \\
    &\implies U(E) > \hat{U}(E)
\end{align}
This proves that fragmented utility estimation is a \emph{systematic lower bound} on the true utility when evidence is synergistic.
\end{proof}

\paragraph{Consequence for Chunk Size Selection.}
Proposition~\ref{prop:underestimation} rigorously justifies the catastrophic performance collapse at small chunk sizes on multi-hop tasks. The fragmentation gap $\Delta_{\text{frag}}$ is not merely noise but a \emph{structured bias} scaling with inter-chunk synergy. The optimal $S^*$ must satisfy $S \ge E_{\mathcal{T}}$ to ensure synergistic evidence is not partitioned.

\paragraph{Connection to Cognitive Science.}
This parallels the \emph{binding problem} in cognitive neuroscience: a reasoning chain stored in working memory must be represented as a unified chunk to preserve logical coherence. Fragmenting it destroys emergent meaning, analogous to how fragmenting a sentence into words loses compositional semantics.

\section{Theoretical Justification for CPMI vs. Surprisal}
\label{app:cpmi_vs_surprisal}

A natural alternative to our CPMI-based utility is a simpler, non-uniform gating policy based solely on \emph{Surprisal} (or Perplexity), i.e., prioritizing chunks with high $\mathcal{L}_{\text{local}}(x) = -\log P(x \mid x_{\text{local}})$. While intuitively appealing (allocating compute to ``hard'' regions), this approach is mathematically suboptimal for long-context adaptation.

\paragraph{Proposition.} Surprisal measures \emph{difficulty}, whereas CPMI measures \emph{dependency}.

Let the information content of a token $x_t$ be decomposed as:
\begin{equation}
\begin{aligned}
    I(x_t \mid x_{\text{global}}) &= I(x_t \mid x_{\text{local}}) \\
    &\quad - \underbrace{I(x_t ; x_{\text{distant}} \mid x_{\text{local}})}_{\text{CPMI}}
\end{aligned}
\end{equation}
High surprisal ($I(x_t \mid x_{\text{local}})$ is large) can arise from two sources:
\begin{enumerate}
    \item \textbf{Aleatoric Uncertainty:} The token is inherently unpredictable (e.g., a random name or number), regardless of context.
    \item \textbf{Epistemic Uncertainty (Contextual):} The token is predictable given long-range context but unpredictable locally.
\end{enumerate}
Gradient adaptation on Type 1 tokens is wasteful—it forces the model to memorize noise. Adaptation on Type 2 tokens is high-value—it retrieves recoverable information.
CPMI ($|\log P_{\text{full}} - \log P_{\text{local}}|$) specifically isolates Type 2 uncertainty by cancelling out the intrinsic difficulty. A surprisal-based baseline would confuse noise with signal, allocating budget to intrinsically hard but context-irrelevant tokens. Thus, CPMI is the theoretically correct objective for \emph{context-dependent} memory consolidation.

\section{Detailed Related Work}

\paragraph{Context Engineering}
Context engineering optimizes information payloads for large language models, spanning techniques from foundational prompt design to advanced management strategies \citep{Mei2025SurveyContext, mei2025survey, Huang2025DirectedInformation, Hua2025ContextEngineering, Sahoo2024SystematicSurvey, Haider2024PromptingAnd, pang2025large}. Core methodologies include in-context learning and chain-of-thought prompting for adaptive reasoning \citep{Weng2024NavigatingThe, Allingham2023SimpleZero, Brown2020LanguageModelsa, Chen2025UsingLlms, Wei2022ChainThought, Ohalete2025CostarPrompting, ge2025innate}, with recent extensions to multimodal agentic in-context learning \citep{fu2025contextnav}. Recent approaches focus on retrieval-augmented generation (RAG) and automated prompt optimization \citep{Lewis2020RetrievalAugmenteda, Nogueira2019PassageRanking, Shin2020ElicitingKnowledge, Li2025HaystackEngineering, Karpukhin2020DensePassage, chen2025rethinking}, with advanced retrieval systems incorporating utility-based training through shared context attribution \citep{xu2025training} and fine-grained citation generation evaluation \citep{xu2024aliice}. Efficient context management techniques address quadratic scaling limitations through window extension, compression, and hierarchical processing \citep{Li2023CompressingContext, Vaswani2017AttentionAll, Mao2024LiftImproving, Fu2024DataEngineering, Duan2025DocopilotImproving, Zhu2025SkyladderBetter, Tan2024LlocoLearning, Wang2024ContextFormer, Song2024HierarchicalContext, Zhou2024LlmMapreduce, Hou2024EnhancingAnd, Ratner2022ParallelContext, Zhang2024AttentionEntropy, Sun2025SolopoUnlocking, Wu2025TokmemTokenized, Wang2024ToolgenUnified, quancai-etal-2025-discomp}. Parallel efforts tackle robustness and knowledge conflicts via context-aware decoding, representation engineering, and activation alignment \citep{He2024PositionEngineering, He2024ContextSteering, Govindan2025TemporalAlignment, Park2025ContextRobust, Wang2025ContextEngineering, Zhao2024EnhancingContextual, Longpre2021EntityBased, Khandelwal2025CocoaConfidence, Shi2023TrustingYour, Zhou2023ContextFaithful, Zhao2024SteeringKnowledge, Shen2025QwenlongCprs, Porretta2025LatentConvergence, Katrix2025ContextAware, Jukic2025ContextParametrization, Houlsby2019ParameterEfficient, Park2025EmulatingRetrieval, bi2024context, bi2024factuality, bi2025parameters, ge2023attack}. Recent work also reveals that LLMs exhibit systematic biases, such as over-relying on surface-level naming patterns when identifying drug ingredients \citep{li2025newsnowtabletscontain}, and evaluates the impact of watermarking on visual language models in document understanding tasks \citep{xu2025doeswatermarkingaffectvisual}. Safety-critical research addresses fine-grained safe generation via specialized representation routers \citep{mei2024hiddenguard}, distinguishes misalignment from maliciousness in jailbreak scenarios \citep{mei2024not}, investigates jailbreak attacks through implicit references \citep{wu2024you}, adversarial metaphors \citep{yan2025benign}, and prompt template stealing vulnerabilities in text-to-image models \citep{wu2025vulnerability}, while also rethinking jailbreak evaluation to investigate real misuse threats \citep{yan2025confusion} and exploring global subspace approaches for LLM detoxification \citep{duan2026projecting}. Additionally, structured output approaches improve question reformulation and cross-lingual summarization through meta-generation \citep{li2024drs, li2024think}. Research on LLM reasoning capabilities includes new concept comprehension through slang understanding \citep{mei2024slang}, graph descriptive order effects on graph problem solving \citep{ge2024can}, and scalable link prediction with LLMs \citep{bi2024lpnl}. Domain-specific applications demonstrate these methods in tabular analysis, translation, clinical NLP, and traffic prediction \citep{Hollmann2023LargeLanguage, Wu2023ExploringPrompt, Sivarajkumar2024EmpiricalEvaluation, Zheng2024LargeLanguage, Yang2023SupervisedKnowledge}, often utilizing filtering mechanisms to enhance retrieval precision \citep{Shi2025ConceptThan, Cheng2024XragExtreme, Chakraborty2025PrincipledContext}, with emerging work on learning to refine pre-training data at scale \citep{bi2025refinex}.

\paragraph{Memory Management}
Memory management in large language models has evolved to address the critical bottlenecks imposed by model weights and, increasingly, the ephemeral activation memory required for inference. While parametric memory stores implicit knowledge in model weights \citep{Li2025MemosOperating, Zhang2025MemoryLarge, Hsieh2023DistillingStep}, the key-value (KV) cache has emerged as the dominant constraint, often consuming significantly more memory than parameters and scaling linearly with sequence length \citep{Liu2023ScissorhandsExploiting, Zhang2024DiffkvDifferentiateda, Kwon2023EfficientMemoryb, Dai2024CormCache, Cai2025NqkvCachea, Roy2025CarCachea, Shutova2025CacheYoua, Barua2024ExploringAutonomous, Modarressi2023RetLlm, Li2025MemosMemory, Cheng2025EccoImproving, Brown2020LanguageModels, Gao2018LowLatency, Dao2022FlashattentionFast}. To mitigate these limitations, research has introduced virtual memory and paging techniques akin to operating systems \citep{Kwon2023EfficientMemoryb, Xue2024NinjallmFast, Koilia2024HardwareAcceleration, Zhang2025JengaEffective, ChittyVenkata2025PagedevictionStructured}, as well as advanced compression and quantization methods to reduce footprints without quality loss \citep{Xie2025ReimaginingMemory, Srinivas2025ScalingTest, Zhang2024LorcLow}. Dynamic eviction strategies, including heavy-hitter identification and attention-based pruning, selectively discard less critical tokens to maintain fixed budgets \citep{Yuan2025WeightedkvAttentiona, Zhang2023HeavyHitter, Zeng2025LetheLayer, Shinwari2025MemoryAugmenteda, Yi2024MethodFor}. Parallelization and offloading frameworks further optimize resource utilization across devices \citep{Yang2024ProtrainEfficienta, Ren2021ZeroOffload, Sheng2023HighThroughputa, Wu2024LayerCondensed, Banasik2025MemoryAccessa}, while hierarchical and biologically inspired architectures offer structured approaches to long-term memory and personalization \citep{Fang2024UnimemTowards, Rezazadeh2024FromIsolated, Kang2025MemoryAgenta, Hou2024SynapticragEnhancing, Li2024CmtMemory, Metinov2025AdaptiveSoft, Zhong2023MemorybankEnhancing, Huang2024EmotionalRag}. Complementary work on knowledge editing addresses how to efficiently update stored knowledge in model parameters, with methods ranging from adaptive token biasing to structured output acceleration \citep{bi2024adaptive, bi2024struedit, bi2025decoding}, while recent studies reveal that related knowledge perturbation matters when editing multiple pieces of knowledge for the same subject \citep{duan2025related}.

\paragraph{Agent Memory}
Research on agent memory in large language models addresses the constraints of limited context windows and static training data by enabling persistent information storage and retrieval \citep{hu2026memoryageaiagents}. Existing approaches categorize memory into distinct types, differentiating between short-term working memory and long-term storage \citep{Han2024LlmMulti, Wang2025MirixMultia, Sridhar2023CognitiveNeuroscience}, as well as distinguishing parametric memory in model weights from explicit contextual memory \citep{Du2025RethinkingMemorya, Shan2025CognitiveMemory}. More granular taxonomies identify specialized forms such as episodic, consensus, semantic, and procedural memory \citep{Han2024LlmMulti, Terranova2025EvaluatingLong}. Structurally, methods range from knowledge-organization and retrieval-oriented mechanisms to architecture-driven hierarchies \citep{Kang2025MemoryAgent, Zhong2023MemorybankEnhancinga}. While early work relied heavily on Retrieval Augmented Generation (RAG) and static vector databases \citep{Lewis2020RetrievalAugmenteda, Xu2025MemAgentica, Vishwakarma2025CognitiveWeave}, recent advancements explore graph-based systems \citep{Anokhin2024ArigraphLearning, Vishwakarma2025CognitiveWeave} and diverse formats like natural language or embeddings \citep{Wang2023SurveyLarge, Shinn2023ReflexionLanguage}. Contemporary frameworks implement core operations—consolidation, updating, indexing, and forgetting—to manage memory dynamics, often incorporating biologically inspired mechanisms like forgetting curves \citep{Xiong2025HowMemory, Park2023GenerativeAgents, Zhao2023ExpelLlm, Cao2025RememberRefine}. Despite these advances, challenges remain in autonomous memory management, with many systems relying on manual predefinitions \citep{Zhang2025LearnMemorize, Wang2025MemLearninga, Zhang2024SurveyThe} or suffering from structural rigidity and catastrophic forgetting \citep{Xu2025MemAgentica, Zhang2025MemgenWeaving, Guo2023EmpoweringWorking}, motivating meta-evolutionary approaches to dynamically optimize agent memory systems \citep{zhang2025memevolvemetaevolutionagentmemory}. Applications span gaming, dialogue, and procedural tasks \citep{Hu2024SurveyLarge, Zeng2024TheStructural, Hu2024HiagentHierarchical, Fang2025MempExploring, Mohammed2025TowardsStandardized}, with emerging work on environment-free exploration for GUI agents \citep{zheng2025vem}, adaptive workflow optimization via meta-learning \citep{zhu2025adaptflow}, cross-modal knowledge graphs for cost-effective game agents \citep{fu2025vistawise}, reality-aligned evaluation for agentic search \citep{xu2025ravine}, coarse-to-fine frameworks for robust GUI grounding \citep{li2025generalist}, and scaling issue-resolving capabilities across repositories \citep{wang2025swemirrorscalingissueresolvingdatasets}. Future research directions emphasize unified theoretical frameworks \citep{Wu2025FromHuman}, episodic memory for single-shot learning \citep{Pink2025PositionEpisodic}, and scalable, autonomous systems capable of self-evolution \citep{Salama2025MeminsightAutonomous, Zhang2024SurveyThe, Hu2025EvaluatingMemory, Huang2025LicomemoryLightweight}.

\paragraph{Test-Time Adaptation.}
A growing line of work treats deployment as an optimization phase, improving robustness and adaptation by updating model states at inference time under shift or novelty, often via self-supervised objectives and safeguards against forgetting or noisy updates \citep{Niu2022EfficientTMD,Tang2023NeuroModulatedHLA,Park2024T4PTTB,Zhang2024TestTimeTOC}. Within LLMs, inference-time optimization spans gradient-based test-time training on task instances or auxiliary data, including adaptation driven by retrieved neighbors, contextual streams, active/verification-guided sample selection, selective test-time learning for evaluation models, efficient rubric-based generative verification for search-augmented systems, and sample-specific language model optimization \citep{hardt2024testtimetrainingnearestneighbors,muhtar2024streamadapterefficienttesttime,Hbotter2024EfficientlyLAE,huebotter2025efficientlylearningtesttimeactive,Moradi2025ContinuousSOF,Jwa2025BecomingEJC,akyurek2025surprisingeffectivenesstesttimetraining,ma2025efficient,hu2025slotsamplespecificlanguagemodel}. In parallel, test-time compute can be optimized at the policy level—allocating computation across candidate solutions or update actions—through principled test-time scaling, meta-learned compute control, reinforcement learning at inference, or environment-augmented generation that achieves steep scaling laws \citep{Snell2024ScalingLTB,Qu2025OptimizingTCA,Zuo2025TTRLTRD,mei2025a1}. Recent work on reward and rubric-based guidance promotes exploration to improve multi-domain reasoning \citep{bi2025reward}. Recent theoretical work has also drawn connections between LLMs and Solomonoff induction, providing foundations for understanding test-time computation \citep{wan2025large}, investigates latent reasoning in LLMs as vocabulary-space superposition \citep{deng2025latent}, and analyzes circular reasoning patterns that cause self-reinforcing loops in large reasoning models \citep{duan2026circular}. Systematic surveys categorize reasoning LLMs from System 1 to System 2 paradigms \citep{li2025system}. Graph-based analysis provides insights into reasoning patterns of LLMs \citep{xiong2025mapping}, divide-and-conquer strategies enhance character-level manipulation \citep{xiong2025enhancing}, and pattern mining approaches mitigate LLM overthinking via early reasoning exit \citep{wei2025stop}. Audio chain-of-thought reasoning extends these paradigms to multimodal settings \citep{xiong2025thinking}. Complementary inference-time strategies reduce effective context cost by compressing or distilling long inputs and intermediate representations, ranging from selective augmentation/compression in retrieval pipelines to learned and training-free prompt compression, dynamic allocation of soft tokens, activation-based beacons, near-lossless KV compression, and importance sampling-based prompt compression \citep{xu2023recompimprovingretrievalaugmentedlms,chevalier-etal-2023-adapting,jiang2024longllmlinguaacceleratingenhancingllms,pan2024llmlingua2datadistillationefficient,fei2025efficientpromptcompressionevaluator,chen2025dastcontextawarecompressionllms,zhang2024longcontextcompressionactivation,chari2025kvdistillnearlylosslesslearnable,chen2025pis}.

\paragraph{Extended Applications and Broader Techniques}
The core principles underlying context engineering and test-time adaptation extend naturally to diverse application domains. Retrieval-augmented systems face challenges including hidden biases that can undermine performance across modalities \citep{yao2025spotlight}. Research on model robustness reveals vulnerabilities to various input manipulations such as adversarial text arrangements \citep{li2024vulnerability} and disambiguation challenges in recognition tasks \citep{li2025texture}, with studies examining how models handle semantic understanding requiring global reasoning \citep{li-etal-2025-semvink}. Attention-based refinement mechanisms have been applied to structured output generation such as LaTeX conversion \citep{li20252r}, controllable generation via diffusion models \citep{xiong2025unveiling}, and contrastive attention for enhanced reasoning \citep{ge2025focusing}. Interface understanding benefits from test-time scaling approaches \citep{wu2025dimo} and structured attention mechanisms for document processing \citep{liu2025structured}. Temporal reasoning advances through frame-interleaved approaches via reinforcement learning \citep{ge2025framemind}, while domain-specific understanding is enhanced through foundational skill evaluation \citep{wu2025refineshot}. Temporal bias analysis reveals unique challenges in sequential data processing \citep{yao2025not}. Dynamic scene modeling benefits from semantic-guided control \citep{chen2025tokensnodessemanticguidedmotion} and hierarchical flow-guided representations \citep{chen2025haif}. Efficient rendering techniques include frequency-importance methods for real-time processing \citep{chen2024frequency}, non-photorealistic rendering for stylized outputs \citep{hu2024real}, and zero-shot synthesis from textual descriptions \citep{zhang2024translating}. Knowledge distillation enables efficient model deployment for domain-specific tasks \citep{yuyao2022vision}. Signal processing techniques bridge different data modalities through reconstruction \citep{fu2025brainvis}, while generative priors enable quality assessment \citep{fu2024dp}. Long-form content generation extends to creative domains with scaled foundation models \citep{yuan2025yue}, and emerging surveys explore new programming paradigms with LLMs \citep{ge2025survey}.


\section{Datasets}
\label{sec:datasets}

We evaluate \textsc{Gdwm} on two complementary long-context benchmarks: \textbf{ZeroSCROLLS}---a \emph{zero-shot} suite adapted from SCROLLS with reliable, task-specific automatic metrics---and \textbf{LongBench v2}---a \emph{realistic} long-context benchmark using multiple-choice questions for robust evaluation.

\subsection{ZeroSCROLLS}
\label{subsec:zeroscrolls}

\paragraph{Benchmark overview.}
ZeroSCROLLS is a zero-shot benchmark for long-text understanding that contains \emph{no training split} and only small validation sets, with each task capped at 500 examples to keep evaluation affordable.
It extends SCROLLS by adapting six long-document tasks and adding four new tasks, covering summarization, question answering, and information aggregation.

\paragraph{Tasks used in this work.}
Following prior long-context evaluation practice, we select six representative ZeroSCROLLS tasks spanning both \emph{sparse-evidence} and \emph{dense-coverage} regimes:
(i) \textbf{GovReport} and \textbf{QMSum} for (query-based) summarization,
(ii) \textbf{Qasper} and \textbf{NarrativeQA} for long-context QA,
(iii) \textbf{QuALITY} (denoted as \textbf{Quality} in our paper) for multiple-choice comprehension,
and (iv) \textbf{MuSiQue} for multi-hop QA.

\paragraph{Evaluation metrics.}
ZeroSCROLLS uses task-aligned automatic metrics: ROUGE for summarization, F1 for extractive/abstractive QA-style tasks, and Accuracy for multiple-choice QA.
We report the official metrics with the benchmark's evaluation scripts.

\begin{table}[t]
    \centering
    \resizebox{0.5\textwidth}{!}{%
    \begin{tabular}{lccr}
        \toprule
        \textbf{Task} & \textbf{Type} & \textbf{Metric} & \textbf{Avg \#Words} \\
        \midrule
        GovReport  & Summarization & ROUGE    & 7{,}273 \\
        QMSum      & QB-Summ       & ROUGE    & 10{,}839 \\
        Qasper     & QA            & F1       & 3{,}531 \\
        NarrativeQA& QA            & F1       & 49{,}384 \\
        QuALITY (Quality) & MC-QA  & Accuracy & 4{,}248 \\
        MuSiQue    & QA (Multi-hop)& F1       & 1{,}749 \\
        \bottomrule
    \end{tabular}%
    }
    \caption{ZeroSCROLLS tasks used in this work and their official metrics/statistics \citep{shaham2023zeroscrollszeroshotbenchmarklong}.}
    \label{tab:zeroscrolls_tasks}
\end{table}

\subsection{LongBench v2}
\label{subsec:longbenchv2}

\paragraph{Benchmark overview.}
LongBench v2 is a realistic long-context benchmark designed to test \emph{deep understanding and reasoning} over long inputs.
It contains \textbf{503} challenging \textbf{multiple-choice} questions with contexts ranging from \textbf{8k to 2M words}, organized into \textbf{6} major categories and \textbf{20} subtasks.
All examples are in English, and each instance includes a long context, a question, four options, a ground-truth answer, and annotated evidence for verification.
The benchmark emphasizes evaluation \emph{reliability} by using accuracy-based scoring (rather than free-form generation metrics).

\paragraph{Task categories and statistics.}
Table~\ref{tab:longbenchv2_categories} summarizes the six categories and their median context lengths.
Notably, LongBench v2 includes long-dialogue history understanding and code-repository understanding, which directly stress \emph{memory} and \emph{retrieval under long context}.

\begin{table}[t]
    \centering
    \resizebox{0.5\textwidth}{!}{%
    \begin{tabular}{lrr}
        \toprule
        \textbf{Category} & \textbf{\#Questions} & \textbf{Median \#Words} \\
        \midrule
        Single-Document QA                & 175 & 51k \\
        Multi-Document QA                 & 125 & 34k \\
        Long In-context Learning           & 81  & 71k \\
        Long-dialogue History Understanding& 39  & 25k \\
        Code Repository Understanding      & 50  & 167k \\
        Long Structured Data Understanding & 33  & 49k \\
        \bottomrule
    \end{tabular}%
    }
    \caption{LongBench v2 category-level statistics reported in \citet{bai2025longbenchv2deeperunderstanding}.}
    \label{tab:longbenchv2_categories}
\end{table}

\paragraph{Evaluation protocol.}
We follow the official LongBench v2 evaluation protocol and report \textbf{accuracy} (overall and category-wise).
For consistency across settings, all methods are evaluated under the same context-length budget of the underlying base model; when an instance exceeds the model limit, we follow the benchmark-recommended preprocessing/truncation behavior \citep{bai2025longbenchv2deeperunderstanding}.

\end{document}